\documentclass[conference]{IEEEtran}

\ifCLASSINFOpdf
\else
\fi
\usepackage{booktabs} 
\usepackage{amsmath, amssymb, bm}
\usepackage[all]{xy}
\usepackage{diagrams}
\usepackage{enumerate}
\usepackage{fancyvrb}
\usepackage{mathtools}
\usepackage{algorithm}
\usepackage{algcompatible}
\usepackage{algpseudocode}
\usepackage{subcaption}
\usepackage{soul}
\usepackage{setspace} 
\usepackage{footnote}
\makesavenoteenv{tabular}

\usepackage[T1]{fontenc}
\usepackage[font=small,labelfont=bf,tableposition=top]{caption}

\DeclareCaptionLabelFormat{andtable}{#1~#2  \&  \tablename~\thetable}

\newarrow{Line} -----
\newarrow{Dash}{}{dash}{}{dash}{}
\newarrow{Corresponds} <--->
\newarrow{Mapsto} |--->
\newarrow{Into} C--->
\newarrow{Embed} >--->
\newarrow{Onto} ----{>>}
\newarrow{TeXonto} ----{->>}
\newarrow{Nto} --+->
\newarrow{Dashto} {}{dash}{}{dash}>

\newcommand{\df}{\displaystyle\frac}

\newcommand{\quotes}[1]{``#1''}
\newcommand{\R}{\mathbb{R}}
\usepackage{color}

\definecolor{lavenderpurple}{rgb}{0.59, 0.48, 0.71}

\newcounter{mycounter}

\usepackage{amsthm}
\newtheorem{Proposition}{Proposition}
\newtheorem*{proof*}{Proof}
\newtheorem{Corollary}{Corollary}

\newtheorem{Lemma}{Lemma}

\begin{document} \sloppy
\title{Lazy stochastic principal component analysis}

\author{\IEEEauthorblockN{Michael Wojnowicz, Dinh Nguyen, Li Li, and Xuan Zhao}
\IEEEauthorblockA{Department of Research and Intelligence, Cylance Inc., Irvine, CA 92612}
\IEEEauthorblockA{ \{firstinitiallastname\}@cylance.com }
}

\maketitle


%
\IEEEpeerreviewmaketitle




\begin{abstract}
Stochastic principal component analysis (SPCA) has become a popular dimensionality reduction strategy for large, high-dimensional datasets.   We derive a simplified algorithm, called Lazy SPCA, which has reduced computational complexity and is better suited for large-scale distributed computation.  We prove that SPCA and Lazy SPCA find the same approximations to the principal subspace, and that the pairwise distances between samples in the lower-dimensional space is invariant to whether SPCA is executed lazily or not.   Empirical studies find downstream predictive performance to be identical for both methods, and superior to random projections, across a range of predictive models (linear regression, logistic lasso, and random forests).  In our largest experiment with 4.6 million samples, Lazy SPCA reduced 43.7 hours of computation to 9.9 hours.  Overall, Lazy SPCA relies exclusively on matrix multiplications, besides an operation on a small square matrix whose size depends only on the target dimensionality.
\end{abstract}

\maketitle

\section{Stochastic Dimensionality Reduction}
Stochastic dimensionality reduction (DR) exploits randomization to scale up traditional techniques to large, high-dimensional datasets.  Stochastic DR may be applied as a preprocessing step before feeding the data into a computationally expensive classifier, e.g. a neural network \cite{dahl}, or it may be directly embedded within algorithms to improve scalability~\cite{influence}.



\subsection{Random Projection (RP)} \label{RP}
Given a dataset $X \in \R^{m \times n}$ of $m$ samples in $n$ dimensions, we perform a \emph{random projection} (RP) to $k<n$ dimensions via 
\begin{equation}  \label{rp}
U = \df{1}{c} X\Omega 
\end{equation}
where $\Omega \in \R^{n \times k}$ is a matrix of random numbers and $c$ is a scalar for norm preservation that depends upon the random projection method used.\footnote{Where possible, vectors in $\R^m$ are denoted by $u$ and vectors in $\R^n$ are denoted by $v$.}${}^{,}$\footnote{Note that a random projection is technically not actually a projection (an endomorphism $X \rTo^P X$ that satisfies $P^2 = P$).}    Random projection  is a computationally cheap technique that approximately preserves pairwise distances between samples with high probability (with error depending on $k$ and $m$). 

There are many methods for constructing the random matrix $\Omega$. A theoretically convenient Gaussian RP takes each $\Omega_{ij}$ as an i.i.d draw from a normal distribution, for example $\Omega_{ij} \overset{i.i.d}{\sim} N(0,1)$.   \emph{Very sparse random projections} \cite{li} save storage and computation by taking each $\Omega_{ij}$ as an i.i.d draw from $\{-1,0,1\}$ with probabilities $\{\tfrac{1}{2p}, 1-\tfrac{1}{p}, \tfrac{1}{2p} \}$ for appropriate choice of $p$. Other variants with similar concentration of measure properties include the subsampled randomized Hadamard transformation (also known as a fast Johnson-Lindentrauss transform)~\cite{ailon} and feature hashing~\cite{hashing}. 

\subsection{Stochastic Principal Component Analysis (SPCA)} \

\begin{table*}
   \caption{Commonly used symbols and terms}
  \begin{tabular}{ll}
 $\widehat{X}$ &  Approximation to dataset $X$, with general form $\widehat{X} = f'f^TX$ ($f'$ is a pseudo-inverse for $f^T$)\\
  $\widehat{\mathcal{I}}$ & Approximation to $\text{im}(X)$, formed via random projection\\
$QQ^TX, \; U'U^TX$ & Matrix implementation of $f'f^TX$ with and without orthonormal basis for $\widehat{\mathcal{I}}$\\
 $V^s, V^\ell$ & Matrix of approximate right singular vectors formed by SPCA and Lazy SPCA, respectively\\
$\mathcal{P}$ &Principal subspace (spanned by $k$ dominant principal component directions)\\
$\widehat{\mathcal{P}}$ & Approximation to $\mathcal{P}$, given by span of columns of either $V^s$ or $V^\ell$\\
$\pi^s, \pi^\ell$ & Dimensionality reduction maps for SPCA and Lazy SPCA, respectively\\
$e_i, id_S, \langle \cdot \rangle, X^{-1}(S)$ & $i$th element of the standard basis, identity on set S,  taking the span, preimage of $S$ under $X$ \\
\end{tabular}
  \label{terminology}
\end{table*}

Principal component analysis (PCA) is a classical linear dimensionality reduction strategy.   Given dataset $X$, one finds a $k$-dimensional subspace $\mathcal{P}_\text{PCA}$ (the \emph{principal subspace}) on which projection of the data has the largest possible variance.  Stochastic principal component analysis (SPCA)~\cite{halko} works similarly, but it uses randomization to find an approximation $\widehat{\mathcal{P}} \approx \mathcal{P}_\text{PCA}$.\footnote{For common symbols and terms, see Table~\ref{terminology}.}    SPCA has a greater computational cost than RP, but because PCA satisfies well-known optimality criteria, one might expect $\widehat{\mathcal{P}}_{\text{SPCA}}$ to better approximate $\mathcal{P}_\text{PCA}$, thereby producing a \quotes{better} dimensionality reduction than RP.   As a result, SPCA has become widely used, implemented in popular libraries by MATLAB~\cite{matlab}, scikitlearn~\cite{scikitlearn}, Apache Mahout~\cite{mahout}, Facebook~\cite{facebook}, and others.\footnote{The same algorithm may be called stochastic/randomized PCA or stochastic/randomized SVD.   See the last paragraph in this section.}   


SPCA begins by solving the \emph{approximate (low-rank) matrix decomposition} (AMD) problem \cite{martinsson}, \cite{kumar}: Given a matrix $X$, a target \text{rank} $k$ and a number $l \geq k$,\footnote{Typically $l=k+p$, where p is a small \emph{oversampling parameter.}} we seek to construct a matrix $Q$ with $l$ orthonormal columns  such that
\begin{equation} \label{traditionalapproximation}
||X - QQ^T X|| \approx \min_{A: \, rank(A) \leq k} ||A - X||
\end{equation}  
Most commonly, the matrix $Q$ is found by using the RP in (\ref{rp}) to produce a matrix $U$ whose columns lie within $\text{im}(X)$ (and in fact approximate $\text{im}(X)$ well~\cite{martinsson}), and then orthonormalizing $U$, e.g. via QR decomposition.  The resulting matrix approximation is provably good; for example, Theorem 1.1 of \cite{martinsson} states that when a Gaussian RP is used in (\ref{rp}), the approximation satisfies
\begin{equation}\label{spca_error}
E||X -QQ^TX|| \leq \bigg[ 1+ \df{4\sqrt{l}}{l-k-1} \cdot \sqrt{min \{m,n\}} \bigg] \sigma_{k+1}
\end{equation}
where $\sigma_{k+1}$ is the $(k+1)$st largest singular value of $X$.

Thus, the approximation $\widehat{X} = QQ^TX$ is computationally useful because, on one hand, it lies within a small polynomial factor of the minimum possible error $\sigma_{k+1}$ for a rank-$k$ approximation, and on the other hand, it can be expressed as the product of two factors, $Q$ and $Q^TX$, which are substantially smaller than $X$.  Then, factorizing $\widehat{X}$ yields an approximate factorization of $X$.   In particular, we can take the SVD of the small factor $Q^TX$ to obtain:

\begin{equation}\label{spca_eqn}
X \approx QQ^TX = Q\tilde{U} \Sigma V^T = U \Sigma V^T
\end{equation}

SPCA implementations~\cite{halko},~\cite{scikitlearn} use the columns of $V$ in (\ref{spca_eqn}) as approximate principal component directions for $X$.  Technically, the right singular vectors of $X$ are its principal components only if $X$ is centered.  However, centering may pose problems for large, sparse datasets $X$, and the SPCA procedure in (\ref{spca_eqn}) will project samples into the same subspace $\widehat{\mathcal{P}}$ of $\R^n$ regardless of whether X is centered.  Moreover, right singular vectors still approximate principal components by adhering to the theory of \quotes{uncentered} principal components~\cite{jolliffe}.  Thus, in the context of this paper, we will refer to \emph{approximate right singular vectors} as \emph{approximate principal components.}  



\subsection{Lazy stochastic principal component analysis (Lazy SPCA)} \
In Proposition~\ref{lazy_fixed_rank_approximations}, we will show that the quality of the low-rank approximation $QQ^T X$ depends only upon the subspace formed by the collection $\{q_1, \dots , q_l\}$.   Thus, we may generalize the construction of $\widehat{X}$ so that there is no need for orthonormalizing vectors when approximating $\text{im}(X)$.  That is, we generalize $X \approx QQ^TX$ in (\ref{traditionalapproximation}) to what we call a \emph{lazily approximated (low-rank) matrix decomposition (Lazy AMD)}. 
\begin{equation} \label{uapproximation}
X \approx U'U^TX
\end{equation}
where the columns of $U$ are not necessarily orthonormal, $U'$ is a pseudo-inverse for $U$, and $U$ and $U'$ will be given in  (\ref{Ut}) and (\ref{U'}).   

Lazy SPCA cheaply obtains a good dimensionality reduction by exploiting Lazy AMD followed by a \quotes{premature truncation} trick (see Section~\ref{SASVD}).  Despite the simplification, LSPCA projects samples to the same subspace of $\R^n$ (Proposition 2) and outputs identical pairwise distances between samples in the new space (Proposition 3), yielding equivalent performance in downstream predictions (Experiments 1 and 2).   This is true even though Lazy SPCA has reduced computational complexity (Proposition 4), substantially reducing run time for many large-scale applications (Experiment 1).    Moreover, the algorithm can now be expressed entirely in terms of easily distributed matrix multiplications, with the exception of a single eigendecomposition of a relatively small $l \times l$ matrix. 

\section{THEORY} \label{THEORY}
\subsection{Overview}

The overview of our argument is described here and reflected in Figure~\ref{theory_overview}. 
\begin{enumerate}
\item We view dataset $X$ as a linear map from $\R^n$ to $\R^m$ with image $\text{im}(X)$. Using a random projection, we can construct a good approximation $\widehat{\mathcal{I}} \approx \text{im}(X)$ by taking $\widehat{\mathcal{I}} $ to be the span of the columns of $U=X\Omega$, where $\Omega$ is a random projection matrix~\cite{martinsson},~\cite{mahoney}. We use $\widehat{\mathcal{I}}$ to construct a low-rank approximation $\widehat{X}$, which maps $\R^n$ to subspace $\widehat{\mathcal{I}} \subset \R^m$ instead of to $\text{im}(X)$.  The approximation $\widehat{X}$ has error bounds given in (\ref{spca_error}).
\item The approximation error depends only on the subspace $\widehat{\mathcal{I}} $, and not on the basis for that subspace.  To see this, we will express $\widehat{X}=f'f^TX$ using constructions\footnote{Here, $f'$ is a pseudo-inverse for $f^T$; the notation generalizes the special case where $f'f^T = QQ^T$ where $Q$ has orthonormal columns.} in (\ref{fT}) and (\ref{fprime}) and show:
\begin{enumerate}
\item  For points in $\R^m$, the operation $f'f^T$ (which maps $X$ to $\widehat{X}$) will leave points in $\widehat{\mathcal{I}} $ unchanged, and will map all points in $\widehat{\mathcal{I}} ^\perp$ to 0 [see (\ref{identity_piece}), (\ref{null_piece})].     Since $\widehat{\mathcal{I}}  \approx \text{im}(X)$ by Step 1, most points $u \in \text{im}(X)$ will be well-approximated by component $u_a$, where $u=u_a+u_b$ and $u_a \in \widehat{\mathcal{I}} , u_b \in \widehat{\mathcal{I}} ^\perp$.
\item So back in $\R^n$, where samples naturally live, points which get mapped to $\widehat{\mathcal{I}} $ by $X$ will be unchanged by the approximation, and points which get mapped to $\widehat{\mathcal{I}} ^\perp$ will get mapped to 0 [see (\ref{identity_piece_in_Rn}), (\ref{null_piece_in_Rn})].   
\end{enumerate}
Thus, the quality of the low-rank approximation $f'f^TX$ to $X$ does not depend on representing $\widehat{\mathcal{I}} $ with an orthonormal basis (see Proposition 1), even though this procedure is commonly done (e.g., ~\cite{halko},~\cite{martinsson},~\cite{scikitlearn},~\cite{mahout}).   
\item The approximate principal subspace for $X$ is the span of the right singular vectors of $f'f^TX$, and this can be obtained by simply taking the right singular vectors of $f^TX$. (See Proposition 2).  
\end{enumerate}

\begin{figure}
  \centering
  \includegraphics[width=\linewidth]{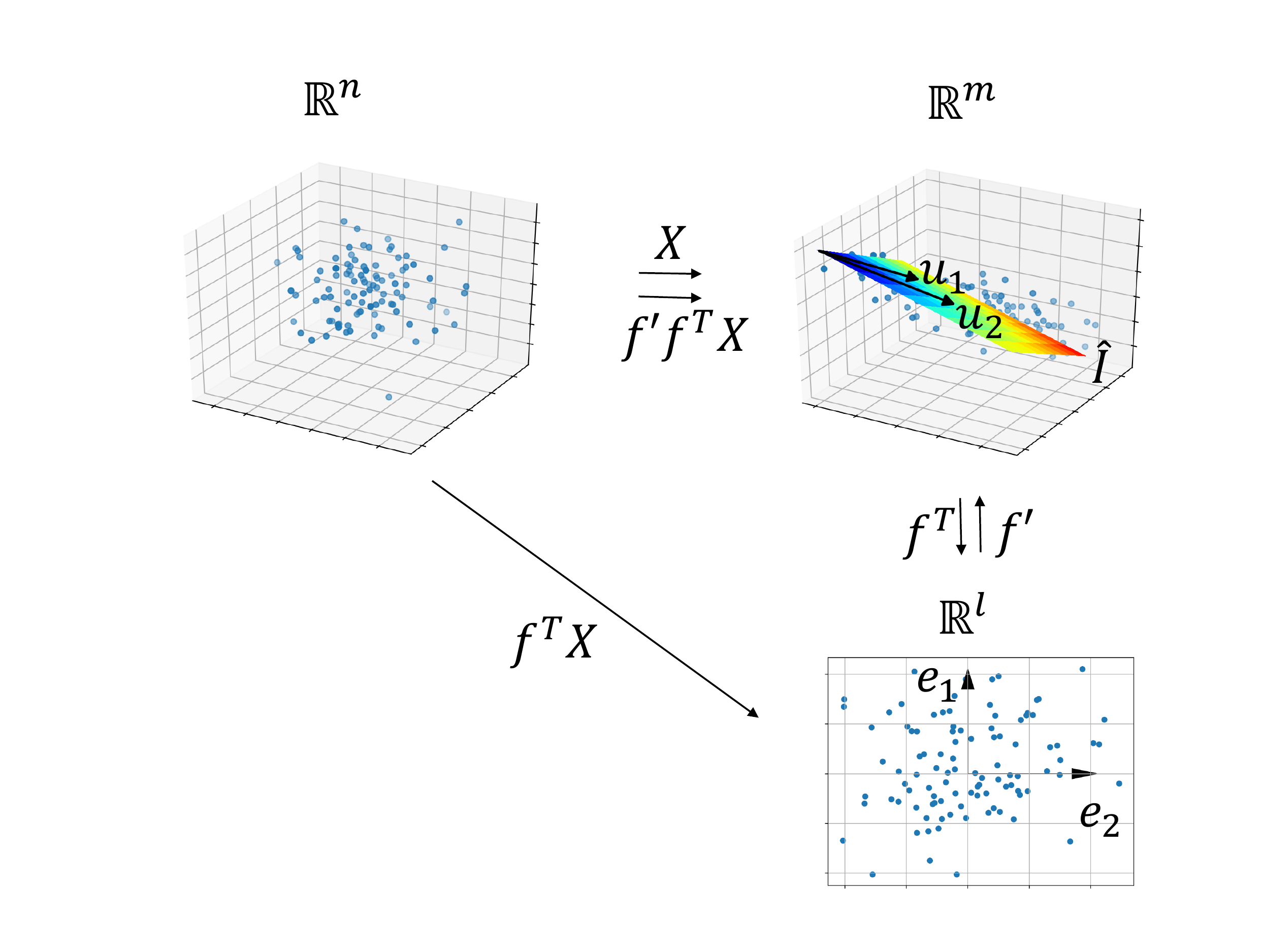}
\caption{An illustration to help motivate Lazy SPCA}
\label{theory_overview}
\end{figure}

\subsection{Lazily approximated low-rank matrix decompositions (Lazy AMD) } \label{low_rank_matrix_approximations}

\subsubsection{Construction of $U$, $U'$} \label{approximate_covers}

Consider the Lazy AMD introduced in (\ref{uapproximation}). We construct $U$ and $U'$ via subspaces $\widehat{\mathcal{I}}$ that approximate $\text{im}(X)$.  Suppose we choose a linearly independent but not necessarily orthonormal collection $\{u_1, \dots , u_l\}$ such that $\widehat{\mathcal{I}} := \langle u_1, \dots , u_l \rangle$ is a subspace of $\text{im}(X)$.   Define 
\begin{align}
\mathbb{R}^m & \rTo^{f^T} \mathbb{R}^l \nonumber  \\
u & \mapsto (u^T_1 u, \dots , u^T_l u) \label{fT}
\end{align}
So $f^T$ has matrix form
\begin{equation}  \label{Ut}
U^T = \left( \begin{array}{ccc} u_{11} & \dots & u_{1m} \\ \vdots & \ddots & \vdots \\ u_{l1} & \dots & u_{lm} \end{array} \right)
\end{equation}

\begin{Lemma} \label{bijective} The map $f^T$ above maps $\langle u_1, \dots , u_l \rangle$ bijectively to $\mathbb{R}^l$.
\end{Lemma}
\begin{proof} Since $\text{rank}(f^T) = \text{rank}(U^T) = \text{rowrank}(U^T) = l$, it follows that $\text{im}(f^T) = \mathbb{R}^l$. So for any $t \in \mathbb{R}^l$ there exists $u \in \mathbb{R}^m$ such that $f^T(u) = t$. Write $u = u' + u''$ where $u' \in \langle u_1, \dots , u_l \rangle$ and $u'' \in \langle u_1, \dots , u_l \rangle^{\perp}$. Then $f^T(u) = f^T(u' + u'') = f^T(u') + f^T(u'') = f^T(u') = t$. So $f^T$ maps $\langle u_1, \dots , u_l \rangle$ surjectively onto $\mathbb{R}^l$. It follows that $\text{ker}(f^t_{| \langle u_1, \dots , u_l \rangle}) =  \{0\}$ and $f^T$ maps $\langle u_1, \dots , u_l \rangle$ injectively to $\mathbb{R}^l$. 
\end{proof}

Thus, there exist $u_1^\prime, \hdots, u_l^\prime \in \widehat{\mathcal{I}}$ such that $f^T(u_i^\prime) = e_i$.  If we define
\begin{align}
\mathbb{R}^l & \rTo^{f^\prime} \mathbb{R}^m \nonumber \\
e_i & \mapsto u_i^\prime \nonumber  \\
\alpha_1 e_1 + \dots + \alpha_l e_l & \mapsto \alpha_1 u_1^\prime + \dots + \alpha_l u_l^\prime \label{fprime}
\end{align} 
then
\begin{align}
f^T f^\prime &= id_{\R^l} \label{updown} \\
f'f^T_{| \langle u_1, \dots , u_l \rangle} &= id_{\langle u_1, \dots , u_l \rangle}  \label{downup}
\end{align}

So $f'$ is a pseudo-inverse for $f^T$.  Note $f'$ has matrix form
\begin{equation} \label{U'}
U' = \left( \begin{array}{ccc} u_{11}' & \dots & u_{1l}' \\ \vdots & \ddots & \vdots \\ u_{m1}' & \dots & u_{ml}' \end{array} \right)
\end{equation}

\subsubsection{Effect of the approximation $f'f^TX$ to $X$ on points in $\R^n$} 

Using $\mathbb{R}^n = \langle u_1, \dots , u_l \rangle \oplus \langle u_1, \dots , u_l \rangle^{\perp}$, we have 
\begin{align}
\text{im}(X) \cap \langle u_1, \dots , u_l \rangle^{\perp} & \neq 0 \quad \text{ if } \widehat{I}  
\subsetneq \text{im}(X) \label{proper_subcover_in_Rm} \\
& =0 \quad \text{ if } \widehat{I} = \text{im}(X) 
\end{align} 

Thus, we describe the operator $f'f^T$ in terms of its action on two complementary subspaces
\begin{equation}
\label{identity_piece}
f'f^T_{| \langle u_1, \dots , u_l \rangle} = id_{\langle u_1, \dots , u_l \rangle}
\end{equation}
and
\begin{equation}
\label{null_piece}
f'f^T_{| \text{im}(X) \cap \langle u_1, \dots , u_l \rangle^{\perp}} = 0
\end{equation}

Back in the domain of the linear map $X$, the operator $f'f^TX$ can again be described in terms of its action on two complementary subspaces:

\begin{align}
f'f^T X_{| X^{-1}(\langle u_1, \dots , u_l \rangle)} &= X_{| X^{-1}(\langle u_1, \dots , u_l \rangle)} \label{identity_piece_in_Rn} \\
f'f^T X_{| X^{-1}(\text{im}(X) \cap \langle u_1, \dots , u_l \rangle^{\perp})} &= 0  \nonumber \\\label{null_piece_in_Rn}
\end{align}
This will be used in Proposition~\ref{lazy_fixed_rank_approximations}. 

Note that in the case $\widehat{I}=\text{im}(X)$, the restriction of the function in (\ref{null_piece}) is over a set that contains only the $0$ vector, and so we obtain an exact approximation
\begin{equation} \label{subcover_approx}
f^\prime f^T X = X
\end{equation}
Otherwise we have an inexact approximation because, via (\ref{null_piece_in_Rn}), we have
\begin{align}
f'f^T X_{| X^{-1}(\text{im}(X) \cap \langle u_1, \dots , u_l \rangle^{\perp})} = 0  \nonumber \\
\quad \quad \neq X_{| X^{-1}(\text{im}(X) \cap \langle u_1, \dots , u_l \rangle^{\perp})}  
\end{align}

\subsubsection{Theoretical results} 

Proposition \ref{lazy_fixed_rank_approximations} states that AMD and Lazy AMD construct identical approximators $\widehat{X}$.

\begin{Proposition}  \label{lazy_fixed_rank_approximations} \emph{(Lazy AMD)}
Let $\{u_1, \dots , u_l\}$ be a linearly independent collection such that $\langle u_1, \dots , u_l \rangle \subseteq \text{im}(X)$. Let  $\{q_1, \dots q_l\}$ be an orthonormalization of that collection.   Let $Q$ be the matrix whose $i^{\text{th}}$ column is $q_i$, $U$ be the matrix whose $i^{\text{th}}$ column is $u_i$, and $U'$ be the matrix defined as in~(\ref{U'}). Then
\begin{equation}
QQ^TX = U' U^T X
\end{equation}
\end{Proposition}

\begin{proof} Since $\langle q_1, \dots , q_l \rangle = \langle u_1, \dots, u_l \rangle$, it follows from (\ref{identity_piece_in_Rn}) and (\ref{null_piece_in_Rn}) that
\begin{align*}
QQ^T &X_{| X^{-1}(\langle u_1, \dots , u_l \rangle)}  = X_{| X^{-1}(\langle u_1, \dots , u_l \rangle)} \\
 & = U'U^T X_{| X^{-1}(\langle u_1, \dots , u_l \rangle)}
\end{align*}
and
\begin{align*}
QQ^T &X_{| X^{-1}(\text{im}(X) \cap \langle u_1, \dots , u_l \rangle^{\perp})} = 0 \\
 & = U'U^T X_{| X^{-1}(\text{im}(X) \cap \langle u_1, \dots , u_l \rangle^{\perp})}
\end{align*}

Hence $QQ^TX = U' U^T X$.
\end{proof}

Thus, we may now generalize Theorem 1.1 of \cite{halko} and (\ref{spca_error}) to a more general class of approximators $\widehat{X}$ without incurring additional error.\footnote{Corollary \ref{Corollary}, unlike the original theorem, assumes that $U$ is full rank.   In actuality, the entire framework holds even when the collection $\{u_1, \hdots, u_l\}$ is linearly dependent, as we show in a future paper.  For now, we show that, in any case $U = X\Omega$ is full \text{rank} with probability 1.  \begin{Lemma} Let $X$ and $\Omega$ be as in Corollary \ref{Corollary}. Then $U = X \Omega$ is full \text{rank} with probability 1.
\end{Lemma}
\begin{proof} Recall that $l \leq \min \{m, n\}$. Assume the first $k < l$ vectors $v_1, \dots , v_k$ have been chosen and they are linearly independent. Then the subspace $\langle v_1, \dots , v_k \rangle$ has measure 0 in $\mathbb{R}^n$. So $P(v_{k+1} \text{ is linearly dependent on } v_1, \dots , v_k) =P(v_{k+1} \in \langle v_1, \dots , v_k \rangle) = 0$.
\end{proof} }  



\begin{Corollary} \label{Corollary} Let $X$ be a real $m \times n$ non-trivial matrix and $\Omega$ be a $n \times l$ Gaussian random projection matrix where $ 2 \leq l \leq \text{min} \{m,n\}$. Suppose $U = X \Omega$ is full \text{rank} and $U^\prime$ has been constructed as in (\ref{U'}). Then for any $k \leq \text{min}\{l, rank(X)\}$,
\begin{equation}\label{Lazy SPCAerror}
\mathbb{E}||X - U'U^TX|| \leq \bigg[ 1+ \df{4\sqrt{l}}{l-k-1} \cdot \sqrt{min \{m,n\}} \bigg] \sigma_{k+1}
\end{equation}
where $\mathbb{E}$ denotes expectation and $\sigma_{k+1}$ denotes the $(k+1)^{\text{st}}$ singular value of $X$.
\end{Corollary}
\begin{proof} The statement follows immediately from Proposition~\ref{lazy_fixed_rank_approximations} and Theorem 1.1 of \cite{halko}.
\end{proof}

\subsection{Lazy SPCA}  \label{SASVD}




\subsubsection{Procedures}

SPCA constructs $S = \langle q_1, \dots , q_l \rangle$ to approximate $\text{im}(X)$ where the $\{q_1, \dots , q_l\}$ are orthonormal. Thus the low-rank matrix approximation described in  Section~\ref{low_rank_matrix_approximations} is given by  $X_{m \times n} \approx Q_{m \times l} Q^T_{l \times m} X_{m \times n}$. Using this, one obtains an approximate SVD.\footnote{We use a tilde to reflect a temporary computational byproduct, and we use superscripts $s$ and $\ell$ to refer to factors relevant to SPCA and Lazy SPCA, respectively.}
\begin{align}
X_{m \times n} & \approx Q_{m \times l} Q^T_{l \times m} X_{m \times n}  \nonumber \\
 & = Q_{m \times l} \tilde{U}_{l \times l} \Sigma^s_{l \times l} V^{sT}_{l \times n}  \text{ (compact SVD)} \nonumber  \\
 & = U^s_{m \times l} \Sigma^s_{l \times l} V^{sT}_{l \times n}.  \label{approx_svd}
\end{align}

Using Lazy AMD, we construct $\widehat{I} = \langle u_1, \dots , u_l \rangle$ to approximate $\text{im}(X)$ without orthonormalizing the $\{u_1, \dots , u_l\}$.  By Proposition \ref{lazy_fixed_rank_approximations}, we obtain an equally good approximation $X \approx U'U^TX$. From this approximation we can obtain the same approximate SVD as before, but with an alternate pathway: 

\begin{align}
X_{m \times n} & \approx U'_{m \times l} U^T_{l \times m} X_{m \times n} \nonumber  \\
 & = U'_{m \times l} \tilde{U}_{l \times l} \tilde{\Sigma}_{l \times l} V^{\ell T}_{l \times n} \text{ (compact SVD)} \nonumber  \\
 & = U^s_{m \times l} \Sigma^s_{l \times l} \tilde{V}^{T}_{l \times l} V^{\ell T}_{l \times n} \text{ (compact SVD)} \nonumber  \\
 & = U^s_{m \times l} \Sigma^s_{l \times l} V^{sT}_{l \times n} \label{simplified_approx_svd}
\end{align}

By the uniqueness of SVD for  $Q_{m \times l} Q^T_{l \times m} X_{m \times n} = U'_{m \times l} U^T_{l \times m} X_{m \times n} $, the approximate SVDs in the final lines of (\ref{approx_svd})  and (\ref{simplified_approx_svd}) are identical.  However, as we show in Propositions 2 and 3, we may use $V^\ell$, an intermediate product of pathway (\ref{simplified_approx_svd}), instead of $V^s$, to perform dimensional reduction.  We call this idea \emph{premature truncation}.  This  saves a computation (either a  QR in (\ref{approx_svd}) or SVD in (\ref{simplified_approx_svd})) that is expensive and can be cumbersome for distributed computation.\footnote{Also note that, for the purposes of dimensionality reduction, we never need to explicitly form the matrix $U'$.}

\subsubsection{Projecting samples to $\widehat{\mathcal{P}}$} \label{subspaces}


 We define the dimensionality reduction maps 
\begin{align}
\pi^s & : \mathbb{R}^n \rTo \mathbb{R}^l, x \mapsto V^{sT} x \nonumber \\
\pi^\ell &: \mathbb{R}^n \rTo \mathbb{R}^l, x \mapsto V^{\ell T} x \label{dim_reduction_maps}
\end{align}
where $\pi^s$ is the the mapping formed by SPCA using $V^s$ in (\ref{approx_svd}) and $\pi^\ell$ is the mapping formed by Lazy SPCA using $V^\ell$ in (\ref{simplified_approx_svd}). Here we show that both $\pi^\ell$ and $\pi^s$ project samples into the same subspace $\widehat{\mathcal{P}}$ of $\R^n$.\footnote{The term \quotes{project} is used loosely here.  More precisely, since $V^T=V^TVV^T$, we can consider these maps as projecting samples onto $\widehat{\mathcal{P}}$ and then identifying $\widehat{\mathcal{P}}$ with $\R^l$ by using the approximate dominant principal component directions as a basis for $\R^l$.}${}^{,}$\footnote{Note, that the two methods will not, in general, find the same basis for $\widehat{\mathcal{P}}$.} 



\begin{Proposition}\label{approximatesingularvectorsspanthesame} 
The approximate right singular vectors of $X$ given by either the columns of $V^s$ (for SPCA) or $V^\ell$ (for Lazy SPCA) form an orthonormal basis for $\langle X^T(u_1), \dots X^T(u_l) \rangle$ in $\text{im}(X^T)$.
\end{Proposition}
\begin{proof} At the compact SVD step $Q^T_{l \times m} X_{m \times n} = \tilde{U}_{l \times l} \Sigma^s_{l \times l} V^{sT}_{l \times n}$ of (\ref{approx_svd}), write $\Sigma^s= \text{diag}(\sigma_1,\hdots,\sigma_l)$ and 
$$\tilde{U} = \left( \begin{array}{ccc} u_{11} & \dots & u_{1l} \\ \vdots & \ddots & \vdots \\ u_{l1} & \dots & u_{ll} \end{array} \right), \; V^s = \left( \begin{array}{ccc} v_{11} & \dots & v_{1l} \\ \vdots & \ddots & \vdots \\ v_{n1} & \dots & v_{nl} \end{array} \right)$$

Then $X^TQ(u_i) = V^s \Sigma^s \tilde{U}^{T}(\sigma_i^{-1} u_i) = V^s \Sigma^s(\sigma_i^{-1} e_i) = V^s(e_i)  = v_i.$  Hence $v_1, \dots , v_l  \in \text{im}(X^T Q) = \langle X^T Q(e_1), \dots , X^T Q(e_l) \rangle = \langle X^T(q_1), \dots , X^T(q_l) \rangle$.   By dimension count, $\langle v_1, \dots , v_l \rangle = \langle X^T(q_1), \dots , X^T(q_l) \rangle$.

Similarly, at the compact SVD step $U^T_{l \times m} X_{m \times n} = \tilde{U}_{l \times l} \tilde{\Sigma}_{l \times l} V^{\ell T}_{l \times n}$ of (\ref{simplified_approx_svd}), write $\tilde{\Sigma}= \text{diag}(\sigma_1,\hdots,\sigma_l)$ and 
$$\tilde{U} = \left( \begin{array}{ccc} u_{11} & \dots & u_{1l} \\ \vdots & \ddots & \vdots \\ u_{l1} & \dots & u_{ll} \end{array} \right), \; V^\ell = \left( \begin{array}{ccc} v_{11} & \dots & v_{1l} \\ \vdots & \ddots & \vdots \\ v_{n1} & \dots & v_{nl} \end{array} \right)$$

Then $X^TU(u_i) = V^\ell \tilde{\Sigma} \tilde{U}^{T}(\sigma_i^{-1} u_i) = V^\ell \tilde{\Sigma}(\sigma_i^{-1} e_i) = V^1(e_i)  = v_i.$  Hence $v_1, \dots , v_l  \in \text{im}(X^TU) = \langle X^TU(e_1), \dots , X^T U(e_l) \rangle = \langle X^T(u_1), \dots , X^T(u_l) \rangle$.   By dimension count, $\langle v_1, \dots , v_l \rangle = \langle X^T(u_1), \dots , X^T(u_l) \rangle$.

Since $\langle q_1, \dots , q_l \rangle = \langle u_1, \dots , u_l \rangle$, it follows that $\langle X^T(q_1), \dots , X^T(q_l) \rangle = X^T(\langle q_1, \dots , q_l \rangle) = X^T(\langle u_1, \dots , u_l \rangle) = \langle X^T(u_1), \dots , X^T(u_l) \rangle$ and we are done.
\end{proof}


\subsubsection{Pairwise distances after dimensionality reduction}  \label{distances}

Here we show that, while dimensionality reduction will often shrink pairwise distances between samples, the resulting distances will be invariant to whether SPCA is executed lazily or not. 


Let  $V = \left( \begin{array}{ccc} v_{11} & \dots & v_{1l} \\ \vdots & \ddots & \vdots \\ v_{n1} & \dots & v_{nl} \end{array} \right) $ be a matrix whose columns are a collection of orthonormal vectors $\{v_1, \dots , v_l\}$ in $\mathbb{R}^n$ and consider the linear map
\[ V^T: \R^n \to \R^l, v \mapsto V^T v = (V^T_1v, \dots , V^T_lv) \] 

\begin{Lemma} \label{normbeforeandafterbyorthonormal} (norm before and after transformation by orthonormal $V$) Suppose $\{v_1, \dots , v_l\}$ is orthonormal. If $v \in \langle v_1, \dots , v_l \rangle$ then $||V^T v|| = ||v||$, else $||V^T v|| < ||v||$.
\end{Lemma}

\begin{proof} If $v \in \langle v_1, \dots , v_l \rangle$ then $||V^T v||^2  = v^TVV^T v = V^T id_{\langle v_1, \dots , v_l \rangle} v  = V^T v = ||v||^2$.  Else write $v = v' + v''$ for some $v' \in \langle v_1, \dots , v_l \rangle$ and nonzero $v'' \in \langle v_1, \dots , v_l \rangle^{\perp}$. Then $||V^T v||^2 = ||V^T(v' + v'')||^2 = ||V^Tv'||^2 = v'^T VV^T v'  = v'^T id_{\langle v_1, \dots , v_l \rangle} v'  = v'^T v'  = ||v'||^2  < ||v'||^2 + ||v''||^2 = ||v' + v''||^2 = ||v||^2$.
\end{proof}

\begin{Corollary} \label{distancebeforeandafterbyorthonormal} (distance before and after transformation by orthonormal $V$) Suppose $\{v_1, \dots , v_l\}$ is orthonormal. If $v_i - v_j \in \langle v_1 , \dots , v_l \rangle$ then $d(V^T v_i, V^T v_j) = d(v_i, v_j)$, else  $d(V^T v_i, V^T v_j) < d(v_i, v_j)$.
\end{Corollary}
\begin{proof} This follows from the fact $d(V^T v_i, V^T v_j) = ||V^T v_i - V^T v_j|| = ||V^T (v_i - v_j)||$ and Lemma~\ref{normbeforeandafterbyorthonormal}.
\end{proof}

 \begin{Proposition} \label{same_distances} Suppose $\pi^s$ and $\pi^\ell$ are the dimensionality reduction maps for SPCA and Lazy SPCA, as described in (\ref{dim_reduction_maps}). Then for all $v_i, v_j \in \R^n$,
 \[  d(\pi^s v_i, \pi^s v_j) = d(\pi^\ell v_i, \pi^\ell v_j)  \leq d (v_i, v_j) \]
\end{Proposition}

\begin{proof}   By Proposition \ref{approximatesingularvectorsspanthesame}, the columns in $V^s$ and the columns in $V^\ell$ span the same subspace $\langle X^Tu_1, \dots , X^T u_l \rangle$.  Thus, if $v_i - v_j \in \langle X^T Xv_1, \dots , X^T Xv_l \rangle$ then by Corollary \ref{distancebeforeandafterbyorthonormal},
 \begin{equation*} \label{same_distances_1} 
 d(V^{sT} v_i, V^{sT} v_j) = d(V^{\ell T} v_i, V^{\ell T} v_j) = d(v_i, v_j)
 \end{equation*}
 Else write $v_i - v_j = v' + v''$ for some $v' \in \langle X^T Xv_1, \dots , X^T Xv_l \rangle$ and nonzero $v'' \in \langle X^T Xv_1, \dots , X^T Xv_l \rangle^{\perp}$ and 
 \begin{align*}
 d(V^{sT} v_i, V^{sT} v_j) & = d(V^{\ell T} v_i, V^{\ell T} v_j)  \nonumber \\
 &= ||v' || < ||v_i - v_j|| = d(v_i, v_j) \label{same_distances_2}
 \end{align*}
\end{proof}

\subsubsection{Computational complexity} Here we show that Lazy SPCA reduces the complexity of SPCA. 
 \begin{Proposition} SPCA has computational complexity $\mathcal{O}(\text{nnz}(X) l+ (m+n)l^2)$. Lazy SPCA is $\mathcal{O}(\text{nnz}(X)l+ nl^2)$.
\end{Proposition}

\begin{proof}
To obtain $V^s$, SPCA uses pathway (\ref{approx_svd}), whose complexity is determined by $\mathcal{O}(\text{nnz}(X)l)$ sparse matrix multiplication,  $\mathcal{O}(nl^2)$ SVD,  and $\mathcal{O}(ml^2)$ QR decomposition.  To obtain $V^\ell$, Lazy SPCA uses pathway (\ref{simplified_approx_svd}) with premature truncation.  This can be seen as discarding either an $\mathcal{O}(ml^2)$ QR in (\ref{approx_svd}) or an $\mathcal{O}(ml^2)$ SVD in (\ref{simplified_approx_svd}).
\end{proof}


\section{ALGORITHMS} 

\begin{table*}
\centering
  \begin{tabular}{l}
  \hline
{\bf Data}  Dataset $X_{m \times n}$ (for streaming versions, split into $s$ horizontal slices, denoted $X_s$); target dimensionality  $k$; \\random projection matrix $\Omega \in \R^{n \times l}$ where $l \geq k$.\\
{\bf Result} Dimension reduction map $\pi =  V_{k \times n}^T$  \\
\hline 
\multicolumn{1}{c}{{\bf Straightforward implementations}}\\
\noindent \begin{minipage}[t]{.4\linewidth}
\vspace{0pt}
{\bf Algorithm 1:} \emph{SPCA~\cite{halko}} 
\begin{enumerate}
\item Construct $U=X \Omega$ 
\item Orthonormalize via $Q,R = \text{qr}(U)$
\item Form $F= Q^T X$, as in (\ref{approx_svd}).
\item Decompose${}^\dagger$ $F = \tilde{U}_{k,k} D_{k,k} V^T_{k,n}$.
\end{enumerate}
\vspace{2pt}
\end{minipage}
\hfill
\noindent \begin{minipage}[t]{.4\linewidth}
\vspace{0pt}

{\bf Algorithm 2:} \emph{Lazy SPCA}
\begin{enumerate}
\item Construct $U=X \Omega$. 
\item Form $F= U^T X$, as in (\ref{simplified_approx_svd}).
\item Decompose${}^\dagger$ $F = \tilde{U}_{k,k} D_{k,k} V^T_{k,n}$.
\end{enumerate}
\vspace{2pt}
\end{minipage}
\\
\hline 
\multicolumn{1}{c}{{\bf Streaming implementations}}\\
\noindent \begin{minipage}[t]{.4\linewidth}
\vspace{0pt}
{\bf Algorithm 3:} \emph{Streaming SPCA~\cite{dsaa}}
\begin{enumerate}
\item Initialize $U_1=X_1 \Omega$; $\tilde{Q},R=\text{qr}(U_1)$; and $F=U_1^TX_1$.
\item {\bf for} $s  \in \{2, 3, \hdots, S\}$  {\bf do} 
\begin{enumerate} 
\item Construct $U_s = X_s \Omega$. 

\item Update  $F \mathrel{+}= U_s^T X_s$, as in (\ref{approx_svd}).
\item Update R via $\tilde{Q},R = \text{qr}( \begin{bmatrix}  R \;\\ U_s \end{bmatrix} )$.
\end{enumerate}
{\bf end for} 
\item  Update  $F = (R^{-1})^T F$, as in (\ref{approx_svd}).
\item Decompose${}^\dagger$ $F = \tilde{U}_{k,k} D_{k,k} V^T_{k,n}$. 
\end{enumerate}
\vspace{2pt}
\end{minipage}
\hfill
\noindent \begin{minipage}[t]{.4\linewidth}
\vspace{0pt}

{\bf Algorithm 4:} \emph{Streaming Lazy SPCA}
\begin{enumerate}
\item Initialize $U_1=X_1 \Omega$ and $F=U_1^TX_1$.
\item {\bf for} $s  \in \{2, 3, \hdots, S\}$  {\bf do} 
\begin{enumerate}
\item Construct $U_s = X_s \Omega$.
\item Update $F \mathrel{+}= U_s^T X_s $, as in (\ref{simplified_approx_svd}).
\end{enumerate}
{\bf end for} 
\item Decompose${}^\dagger$ $F = \tilde{U}_{k,k} D_{k,k} V^T_{k,n}$.
\vspace{2pt}
\end{enumerate}
\end{minipage} \\
\hline
\end{tabular}
  \caption{ \emph{Implementations (both straightforward and streaming) of SPCA and Lazy SPCA.}    The notation $\text{qr}$ refers to a subroutine for performing QR decomposition.  The tilde notation for the output $\tilde{Q}$ refers to the fact that the matrix is not used and needs not be explicitly computed. ${}^\dagger$: Note that all decompositions of matrix factor $F$ in the final step are done by a truncated SVD (i.e. compute or retain only the $k$ dominant singular vectors and values).  If desired for distributed implementations (e.g.~\cite{mahout}), the SVD can be executed via nothing more than matrix multiplications and the eigensolution of a small $l \times l$ matrix.  Compute the eigensolution $FF^T=U \Lambda U^T$.  Then $D=\Lambda^{1/2}$ and $V=F^TUD^{-1}$, where the exponents of the diagonal matrices refer to elementwise operations.  Using this, note that Algorithm 4 can trivially be extended to the case where $X$ is sliced into both horizontal and vertical blocks.}
\end{table*}

We provide straightforward (i.e., in core) implementations of SPCA  and Lazy SPCA  in Algorithms 1 and 2.   However, stochastic dimensionality reduction is typically performed when $X$ exceeds the size of a computer's core memory.\footnote{Otherwise, if $X$ fits into memory, traditional (deterministic) SVD would be easily applied.}  Thus, we also provide streaming implementations in Algorithms 3 and 4.    Overall, Lazy SPCA is both faster and less cumbersome for large datasets (as all operations can be rendered as matrix multiplications, besides finding the eigensolution of a small $l \times l$ matrix.)





\section{Experiment 1} \label{experiment1}


In Experiment 1, we demonstrate the techniques on a large dataset from the context of automatic malware classification.

\subsection{Data and Method}

This dataset~\cite{dsaa} consists of 4,608,517 portable executable files and determined to be either malicious or clean.   Each file is represented as 98,450 features, mostly binary, with mean density 0.0244. 

We applied three dimensionality reduction methods to this dataset: RP, SPCA and Lazy SPCA.   Across methods, we employed fixed \emph{very sparse random projections} with density set to $log(k)/k$, the aggressive value in~\cite{li}. The target dimensionality $k$ was set to 100, 500, 1000, 5000, 10000 and 20000.    For simplicity, we avoid oversampling and set $l=k$.   

The dataset was divided up into horizontal slices that were represented as (Float32, Int64) Sparse CSC matrices.  The most expensive steps (dense-by-sparse matrix multiplication and QR decomposition) were implemented using the Intel Math Kernel Library (\emph{mkl}).  All computations were performed in Julia v0.3.8 on a single Amazon EC2 r3.8 instance with 16 physical cores (32 hyperthreaded cores) and 244 GB of RAM.\footnote{Note that although the exact timings depend on implementation, the qualitative properties of the results depend on the computational complexity of the algorithms.}
  
A $L^1$-penalized logistic regression (or ``logistic lasso") classifier was trained on 80\% of the samples, randomly selected without replacement.\footnote{The model complexity parameter was fixed at 1 (rather than optimized) to place equal weight on the likelihood term and the penalty term.}   The classifier was tested on the remaining 20\%.

\subsection{Results}

  \begin{figure*}
\begin{minipage}[t]{.48\linewidth}
\vspace{.4in}
\centering
  \begin{tabular}{lccc}
    \toprule
   & \multicolumn{3}{c}{Dimensionality Reduction} \\
 Target Dim. ($k$)   & RP & SPCA & Lazy SPCA \\
    \midrule
100 & 89.63 & {\bf 94.86} & {\bf 94.84} \\
500 & 94.41 & {\bf 97.24} & {\bf 97.24} \\
1,000 & 96.40 & {\bf 97.98} & {\bf 97.98} \\
5,000 & 98.38 & {\bf 98.74} & {\bf 98.74} \\
10,000 & 98.73 & {\bf 98.92} & {\bf 98.92} \\
15,000 & 98.84 & {\bf 98.99} & {\bf 98.99} \\
20,000 & 98.94 & {\bf 99.03} & {\bf 99.03} \\
  \bottomrule
\end{tabular}
\vspace{0pt}
\end{minipage}
\begin{minipage}[t]{.4\linewidth}
\vspace{0pt}
\centering
\includegraphics[width=\linewidth, angle=90]{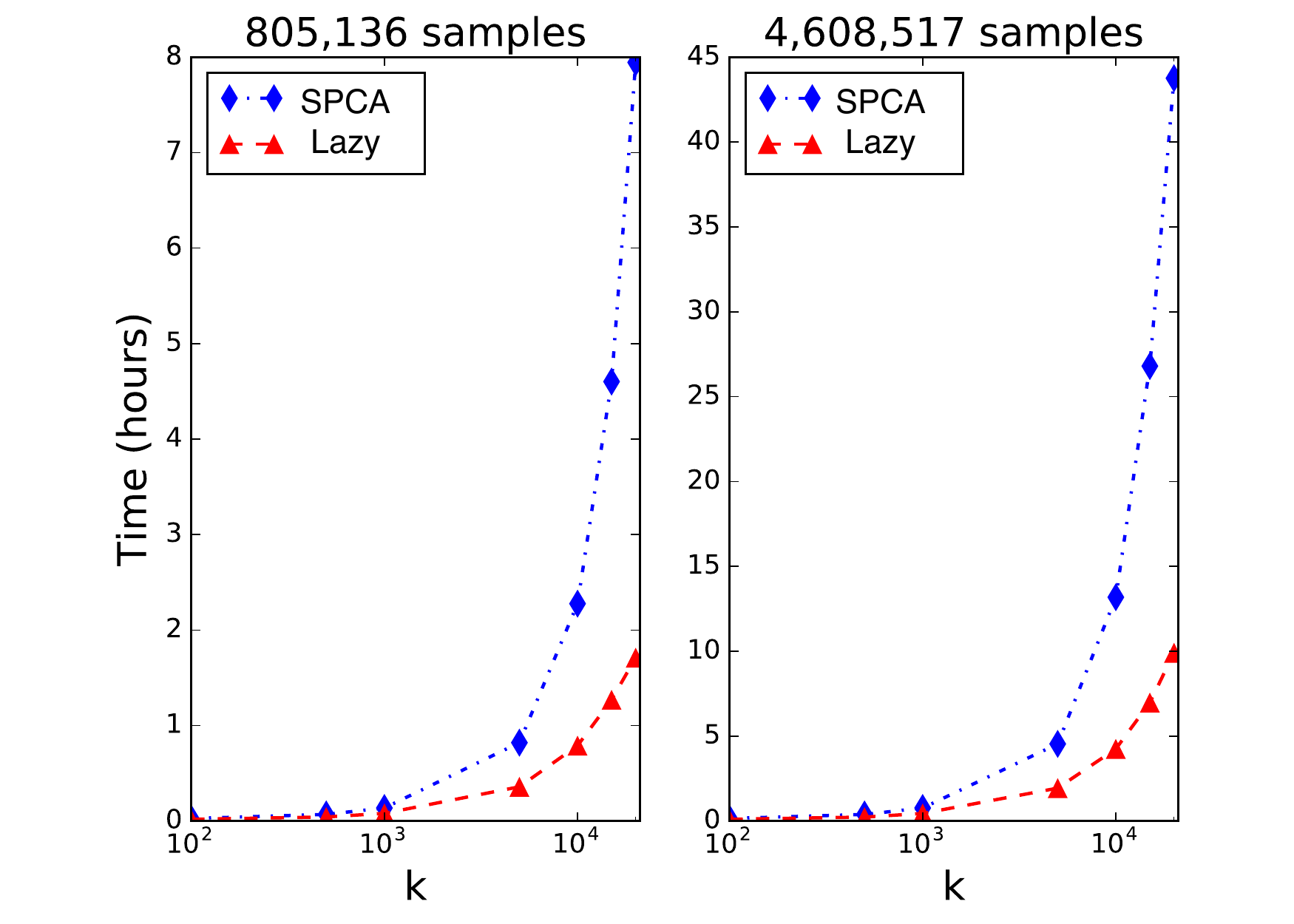}
\end{minipage}%
\caption{The table shows predictive performance (\%) by a L1-penalized logistic regression classifier on a hold-out test set after various large-scale dimensionality reduction methods.  The plot shows run times over the same target dimensionalities.}
\label{expt1results}
  \end{figure*}

In the left table of Figure~\ref{expt1results},  we see that SPCA and Lazy SPCA yield identical downstream predictive performance.  Moreover, both outperform RP for all $k$, although the superiority decreases as $k$ increases.    In the right plots of Figure~\ref{expt1results}, we see that Lazy SPCA is faster than SPCA across all $k$.  This superiority increases as $k$ increases, as expected by Proposition 4.   In the largest data analysis, it took 9.9 versus 43.7 hours to obtain an equivalently useful dimensionality reduction.

\section{Experiment 2} \label{experiment2}

In Experiment 2, we demonstrate the techniques on a smaller yet publicly available dataset, with publicly available code.\footnote{See https://github.com/CylanceSPEAR/lazy-stochastic-principal-component-analysis.}  


\subsection{Data and Methods}

We evaluated our dimensionality reduction methods on the Home Depot Product Search Relevance dataset from the Kaggle competition of the same name.  The goal is to predict the ratings of the relevance of a customer search term to a product.\footnote{For example, one rater might consider a search for "AA battery"  to be highly relevant to a pack of size AA batteries (relevance = 3), mildly relevant to a cordless drill battery (relevance = 2.2), and not relevant to a snow shovel (relevance = 1.3).}    For this study, we generated features by using the co-occurrence TF-IDF of product title and search terms.   The resulting dataset of $74,067$ samples (search term-product pairs) and $28,606$ features has a density of $0.000513$.   Dimensionality reduction was performed as in Experiment 1, except that the RP matrix was constructed using the conservative density $\sqrt{k}$ [10].  Distances between approximate principal subspaces were measured by the chordal distance on the Grassmann manifold, $||V_iV_i^T - V_jV_j^T||_F$, where the columns of $V_i$ form an orthonormal basis for the $i$th subspace. 

\subsection{Results}

 In Figure~\ref{homedepotresults}, we see that SPCA and Lazy SPCA result in identical downstream predictive accuracy across a range of target dimensionalities $k$, and that both outperform RP.  To help explain this, Figure~\ref{subspace} shows that, as expected by Proposition 2, SPCA and Lazy SPCA project samples onto the same $k$-dimensional approximate principal subspace $\widehat{\mathcal{P}} \subset \R^n$ (which is closer to the true principal subspace $\mathcal{P}$ than the subspace found by RP).

\begin{figure}
\begin{subfigure}{.25\textwidth}
  \centering
  \includegraphics[width=\linewidth]{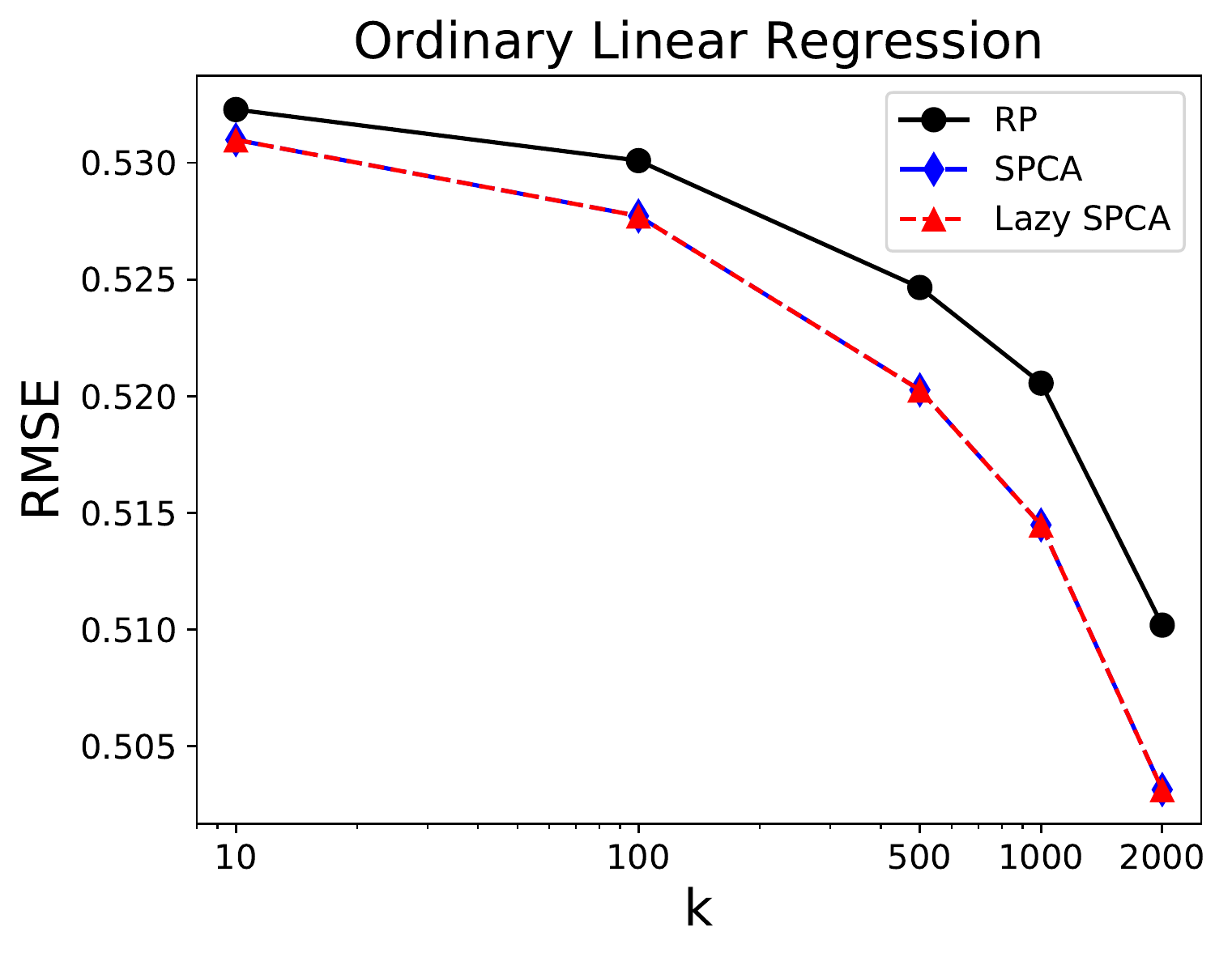}
\end{subfigure}%
\begin{subfigure}{.25\textwidth}
  \centering
  \includegraphics[width=\linewidth]{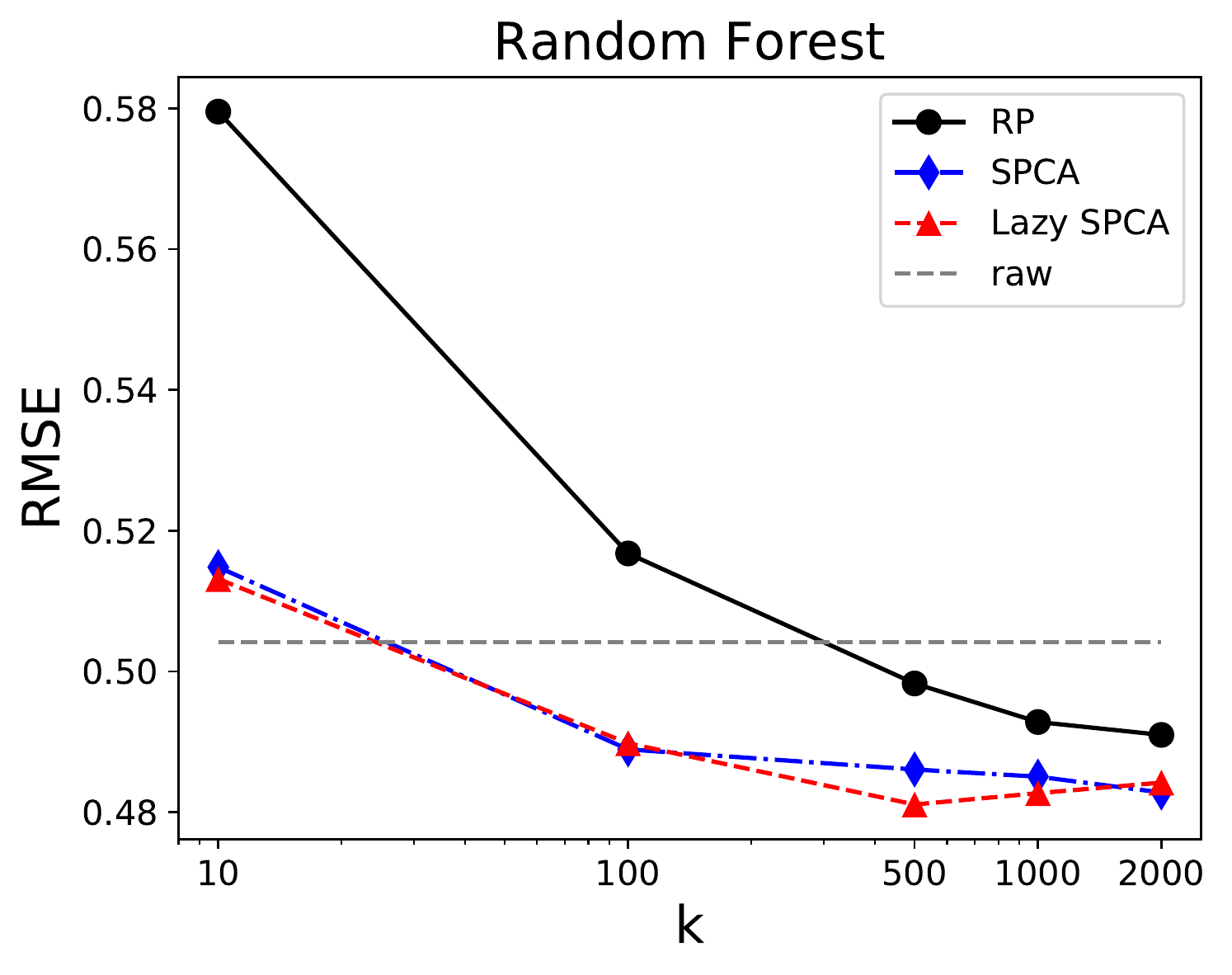}
\end{subfigure}
\caption{ Error in downstream prediction for linear regression (left) and random forest (right) after three different dimensionality reduction strategies.}
\label{homedepotresults}
\end{figure}

\begin{figure}
  \centering
  \includegraphics[width=.5\linewidth]{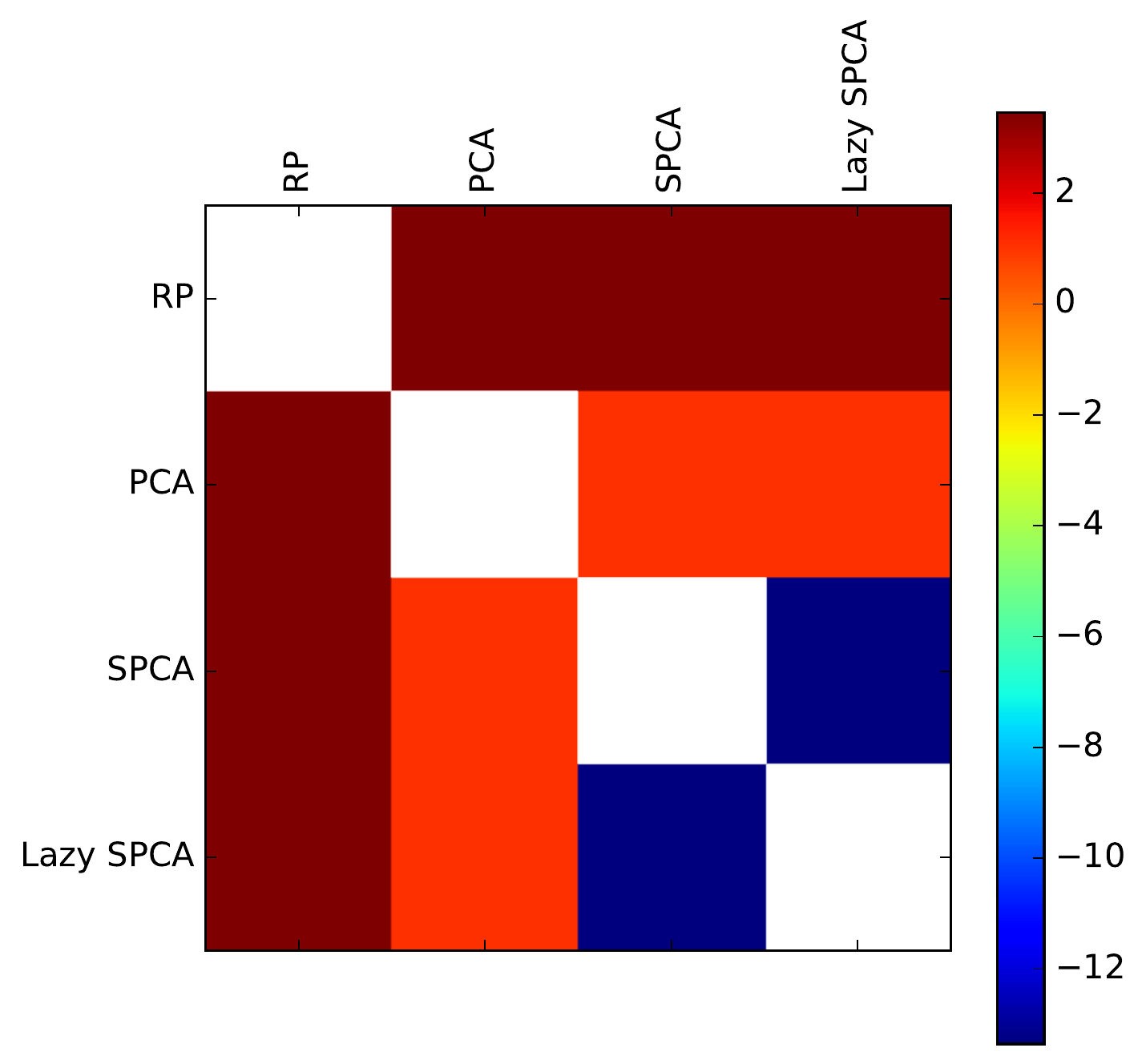}
\caption{ Distances (log scale) between principal subspaces into which samples are projected by the various stochastic dimensionality reduction techniques (here k=100, but the result is typical across $k$). }
\label{subspace}
\end{figure}



\section{Conclusion} \label{conclusion}

We develop a framework for simplifying stochastic principal component analysis when used as a tool for dimensionality reduction. Compared to SPCA, Lazy SPCA is both faster and better suited for distributed computation.  At the same time, it projects samples to the same subspace, yields identical pairwise distances between samples, and results in identical empirical performance in downstream classification. 

\section*{Acknowledgments}
We thank John Hendershott Brock for helpful comments. 
\end{document}